\documentclass[letterpaper, 10 pt, conference]{ieeeconf}
\IEEEoverridecommandlockouts                              

\usepackage{cite}

\usepackage{amsmath,amssymb,amsfonts}
\usepackage{algorithmic}
\usepackage{graphicx}
\graphicspath{ {./figures/} }
\usepackage{textcomp}
\usepackage{xcolor}
\usepackage{bbm}
\usepackage{comment}
\usepackage{paralist}
\usepackage[caption=false]{subfig}
\usepackage{scalerel}    
\usepackage{soul}
\usepackage{amsthm}
\theoremstyle{plain}

\usepackage{multirow}
\usepackage[linesnumbered,ruled]{algorithm2e}

\theoremstyle{remark}
\newtheorem{remark}{Remark}
\theoremstyle{definition}

\theoremstyle{plain}
\newtheorem{theorem}{Theorem}

\newtheorem{problem}{Problem}
\newtheorem{lemma}{Lemma}
  
\newtheorem{observation}{Observation}

\usepackage{dirtytalk}
\usepackage[hidelinks]{hyperref}
\newcommand*{\myDots}{.\kern-0.08em.\kern-0.08em.} 

\let\oldnl\nl
\newcommand{\nonl}{\renewcommand{\nl}{\let\nl\oldnl}}

\newcommand{\ABA}[1]{{\footnotesize\color{blue}[{\bf ABA:} \textsf{#1}]}}

\DeclareMathOperator*{\argmax}{argmax} 
\usepackage{mathtools}
\DeclarePairedDelimiter\abs{\lvert}{\rvert}%
\usepackage{bm}

\DeclarePairedDelimiterX{\norm}[1]{\lVert}{\rVert}{#1}

\usepackage{mathrsfs}
\usepackage[para,online,flushleft]{threeparttable}
\usepackage{array,booktabs,makecell}
\usepackage{diagbox}

\begin{document}
\title{\LARGE \bf
Where to Drop Sensors from Aerial Robots\\ to Monitor a Surface-Level Phenomenon?
\thanks{This work is supported in part by National Science Foundation Grant No. 1943368.}
\thanks{$^{*}$ indicates equal contribution and authors are listed alphabetically}
}

\author{ Chak Lam Shek$^{*}$, Guangyao Shi$^{*}$, Ahmad Bilal Asghar, and Pratap Tokekar
\thanks{University of Maryland, College Park, MD 20742 USA {\tt\small [cshek1}, {\tt\small gyshi}, {\tt\small abasghar}, {\tt\small tokekar]@umd.edu}}
}

%

\maketitle
\thispagestyle{empty}
\pagestyle{empty}

\begin{abstract}

We consider the problem of routing a team of energy-constrained Unmanned Aerial Vehicles (UAVs) to drop unmovable sensors for monitoring a task area in the presence of stochastic wind disturbances. In prior work on mobile sensor routing problems, sensors and their carrier are one integrated platform, and sensors are assumed to be able to take measurements at exactly desired locations. By contrast, airdropping the sensors onto the ground can introduce stochasticity in the landing locations of the sensors. We focus on addressing this stochasticity in sensor locations from the path planning perspective. Specifically, we formulate the problem (Multi-UAV Sensor Drop) as a variant of the Submodular Team Orienteering Problem with one additional constraint on the number of sensors on each UAV. The objective is to maximize the Mutual Information between the phenomenon at Points of Interest (PoIs) and the measurements that sensors will take at stochastic locations. We show that such an objective is computationally expensive to evaluate. To tackle this challenge,  we propose a surrogate objective with a closed-form expression based on the expected mean and expected covariance of the Gaussian Process. We propose a heuristic algorithm to solve the optimization problem with the surrogate objective. The formulation and the algorithms are validated through extensive simulations.  
\end{abstract}

\section{Introduction}

Multi-robot systems have been widely used in scientific information gathering including exploring the ocean \cite{ludvigsen2021collaborating, shi2023robust}, tracking algal blooms \cite{jung2017development}, and monitoring soil \cite{tokekar2016sensor}. The planning problem on this topic is usually named Informative Path Planning (IPP), in which the research focus is on how to design planning algorithms to coordinate multiple robots to collect as much useful information as possible given the limited onboard resources (e.g., sensing and battery). 
In some cases, the robotic platform and the sensors for scientific monitoring are integrated systems and are treated as mobile sensors  as a whole \cite{hollinger2014sampling, ma2018data, tokekar2013tracking}. 
\begin{figure}[ht]
    \centering
    \includegraphics[scale=0.4]{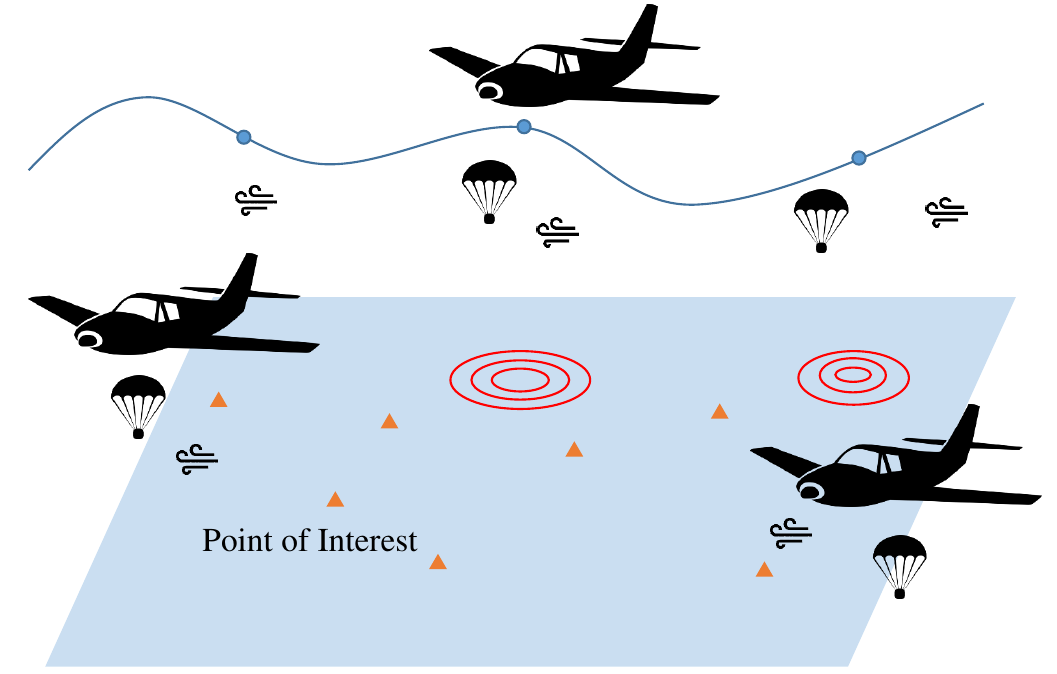}
    \caption{An illustrative example of airdropping sensors. }
    \label{fig:sensor_drop_example}
\end{figure}
In other cases, the robotic platforms are treated as carriers of sensors \cite{stephens2022integrated, marques2015critical, kalaitzakis2021marsupial, tokekar2016sensor, asghar2023risk, shi2022risk}, and they are separable. The research efforts for such cases are mainly devoted to finding collaborative route strategies for these mobile platforms to serve the sensors to finish the sampling tasks. Our research is also along this line and we are interested in how to airdrop sensors to an area of interest with a team of Unmanned Aerial Vehicles (UAVs).

Specifically, we consider the problem of airdropping multiple sensors to the ground with a team of budget-constrained UAVs to reduce the uncertainty of Points of Interest (PoIs) as shown in Fig. \ref{fig:sensor_drop_example}. If the UAVs can precisely drop the sensors to the desired locations, such a problem is closely related to the classic Team Orienteering Problem (TOP) \cite{chao1996team}. However, due to wind disturbances, when we release one sensor from the UAV, its landing location, i.e., the sampling location, is stochastic. This is the main difference from the existing research on mobile robotic sensors, in which authors usually assume that robots can take samples at precisely the desired location. Such a difference requires to rethink of the underlying optimization for planning. 

To this end, we propose a new variant of the TOP  for airdropping sensors with UAVs, in which the stochasticity of the sensor landing position is explicitly considered. However, the resulting optimization objective is computationally expensive to evaluate. To address this challenge, we resort to a Gaussian approximation approach \cite{dallaire2011approximate} to obtain one surrogate objective with one closed-form expression. With this surrogate objective, we show that the problem can be solved in polynomial time and near optimally.

In summary, the main contribution of this paper is:
\begin{itemize}
    \item We propose a variant of the Submodular Team Orienteering Problem to model the sensor dropping problem with aerial robots. 
    \item We propose one computationally efficient surrogate objective function for the proposed problem and propose a heuristic algorithm to solve it. 
    \item We demonstrate the effectiveness of our formulation and algorithm through simulations. 
\end{itemize}

The rest of the paper is organized as follows. We first
give a brief overview of the related work in Section \ref{sec:related_work}. Then, we explain the problem setup and formulation in Section  \ref{sec:problem_statement}. We introduce the technical approach in Section \ref{sec:technical_approach} and validate the formulation and the proposed framework in Section \ref{sec:evaluation}.

\section{Related Work}\label{sec:related_work}
In this section, we present the work most closely related to ours. We first discuss the related work on airdropping sensors, followed by stationary sensor placement and mobile sensor planning, and finally on estimating stationary fields with Gaussian Processes.

\subsection{Airdroping sensors}

Dropping resources from an aerial vehicle has long been of interest, particularly for military and search-and-rescue operations. For example, in military resupply missions, aircrafts are required to accurately deliver supplies to the target areas, taking into account geological factors and weather conditions. Extensive research has been conducted on low-level optimization of the release trajectory to achieve high precision in airdrop operations~\cite{mathisen2020autonomous, gerlach2016precision, leonard2019koopman, leonard2017probabilistic, iyer2020airdropping}. In this work, we focus on the complementary high-level planning of where to drop the sensors from multiple UAVs to monitor a surface-level phenomenon. We abstract the low-level trajectory control by assuming that for any given airdrop trajectory planner,  the associated uncertainty of the landing position of the sensor is known. Specifically, we focus on route-level  planning for multiple UAVs to deploy multiple sensors to the area of interest for environmental monitoring applications.

Our work is closely related to that of Gerlach et al.~\cite{gerlach2016precision}. They formulate the problem of dropping multiple payloads to multiple targets as a Traveling Salesperson Problem (TSP). However, there are two key differences between their work and ours. First, our objective is to reduce the uncertainty at Points of Interest (PoIs) by dropping sensors and we use an information-theoretic metric. In contrast, the objective in \cite{gerlach2016precision} is to minimize the risk encountered by the soldiers. Second, our problem involves multiple energy-constrained UAVs, which cannot be modeled as TSP or its variants.

\subsection{Sensor Placement and Mobile Sensor Planning}
The sensor placement problem aims to maximize the information gain or sensing quality by strategically selecting sensor deployment locations. The typical approach is to model the phenomenon as a Gaussian Process~\cite{williams2006gaussian} and use information theoretic measures for placing the sensors. The foundational work was done by Krause el al.~\cite{krause2008efficient} who showed that the partial monotonicity and submodularity allows a greedy placement to achieve a constant-factor approximation algorithm. This work was later extended to mobile sensor planning (also termed as informative path planning). Binney et al.~\cite{binney2010informative} introduced the additional constraint of identifying a feasible path that connects these selected sensing locations. One approach to finding such paths is to convert the problem into an orienteering instance with submodular rewards. In~\cite{roberts2017submodular} this problem is solved by constructing an additive approximation for the coverage objective to find a UAV path for image acquisition. A recursive greedy algorithm~\cite{chekuri2005recursive} is used in~\cite{binney2013optimizing,singh2007efficient} to solve the submodular orienteering problem for informative path planning. This approach provides guarantees for the submodular objective but
runs in quasi-polynomial time, limiting its use for large problem instances. 

In the context of a multi-robot setting, the orienteering problem can be solved iteratively, where the single robot performance guarantee can be extended to the multi-robot scenario~\cite{singh2007efficient}. Our work closely aligns with this body of work on informative path planning with a key difference. Because we are airdropping sensors, the exact sensing location depends on the wind field and is not known, unlike existing work. We show how to deal with this additional source of uncertainty.

\subsection{GP with Uncertain Inputs}
We use Gaussian Processes~\cite{williams2006gaussian} to model the spatial function that is to be estimated by the sensors. Since we do not know the exact locations the sensors will fall at before planning UAV paths, the input to GP regression is uncertain. It is shown that the predictive distribution for Gaussian processes with uncertain inputs may not be Gaussian in general in~\cite{girard2004approximate}. Various approaches have been used to deal with input uncertainty in GPs. In the Bayesian approach, the distribution with uncertain input locations can be obtained by integrating over the uncertainty of the locations ~\cite{muppirisetty2015spatial}. However, these integrals are analytically intractable in general. Taylor expansion about the uncertain locations is used in~\cite{mchutchon2011gaussian} to present an approximate method that requires the derivative of the mean of $f$. The Gaussian Approximation method~\cite{muppirisetty2015spatial,girard2004approximate,dallaire2011approximate} assumes that the posterior distribution is Gaussian and finds its expected mean and expected covariance by integrating over the uncertainty of the locations. For certain kernel functions, these co-variances can be computed analytically. We employ the Gaussian approximation method in this paper to handle the random sensor locations. 

\section{Problem Statement}\label{sec:problem_statement}
Consider a weighted graph $G = (V, E)$, where the vertex set $V$ represents locations that can be visited by a team of $m$ UAVs. The weight $w(u,v)$ of an edge $(u, v) \in E$ represents the time taken or energy spent by the UAVs to travel from vertex $u$ to vertex $v$. Let $(x_v, y_v, z_v)$ represent the coordinates of vertex $v$. Each vertex corresponds to a location where one of the UAVs can drop a sensor onto the ground below to observe the spatial field.
The sensor's landing position on the surface, denoted by $\textbf{q}_v$, can vary depending on the wind conditions at the drop location $v$ and the height of the drop location $z_v$. We assume that $\textbf{q}_v$ follows a normal distribution, specifically $\textbf{q}_v \sim \mathcal{N}(\bar{\textbf{q}}_v, \Sigma_v)$, and that $\bar{\textbf{s}}_v$ and $\Sigma_v$ are known for each $v\in V$. Each UAV $i\in[m]$ has a given number of sensors $k_i$ and limited amount of time (or energy) $T_i$ to visit some locations in $V$ and to drop the sensors from those locations. The path of UAV $i$ must start and end at its designated depot location $r_i\in V$. 

The purpose of dropping sensors is to observe the value of a spatial function $f$ at specific points of interest (POI) $U$ on the ground. Each sensor obtains a measurement of the underlying field with additive Gaussian noise. Since we may have fewer sensors than POI, and due to the stochastic nature of sensor drop, we will need to estimate the value of $f$ at POI. Consequently, there will be inherent uncertainty associated with these estimates. 
Gaussian Processes associate a random variable with each POI in $U$ and the joint distribution over $U$ can be used to quantify the information gained by the sensors dropped by the UAVs. 

Given paths $P = \{P_1,\ldots, P_m\}$ for the UAVs, let $S(P) = \{S_1,\ldots,S_k\}$ represent the corresponding sensor drop locations, and let $Q(P)$ be the random variable representing the sensor locations, i.e., for every drop location $v\in S$, the sensor location $\textbf{q}_v \in Q$.  Also, let the length of the path $\ell(P_i)$ denote the total time taken by the UAV $i$ to visit all the locations in $P_i$. Let $\eta$ be the time required to drop a sensor. Therefore, the total time of a path $P_i$ is given as $C(P_i)=\ell(P_i) + \abs{S(P_i)}\eta$. 

Let $\mathcal{F}_U$ represent the random variable associated with POI $U$ and let $\mathcal{F}_Q$ represent the random variable associated with sensor readings at locations in $Q$. Then $Pr(\mathcal{F}_U|\mathcal{F}_{Q(P)}=f_Q)$ is the prediction at $U$ given sensor readings at locations in $Q(P)$. To simplify notation, we will use $S$ and $Q$ going forward, without explicitly indicating their dependence on UAV paths $P$. We focus on the \emph{offline} planning problem~\cite{singh2009efficient} where the plan must be decided before dropping any sensor.

The mutual information -- as a function of the UAVs' paths -- between the random variables $\mathcal{F}_U$ and $\mathcal{F}_Q$ is defined as, 
\begin{equation*}
    MI(P) = H(\mathcal{F}_U) - H(\mathcal{F}_U|{\mathcal{F}_Q}),
\end{equation*}
where $H(\mathcal{X})$ represents the entropy of random variable $\mathcal{X}$. We now formally define the multi-UAV sensor drop problem.
\begin{problem}[Multi UAV Sensor Drop]
\label{pbm:main_problem}
Given the points of interest $U$, sensor drop locations in $G=(V, E)$ along with the mean $\bar{\textbf{q}}_v$ and covariance $\Sigma_v$ of sensor's location associated with each $v \in V$, $k_i$ sensors and budget $T_i$ for each UAV $i\in[m]$, find path $P_i$ rooted at the depot $r_i$ along with drop locations $S_i$ for each UAV $i\in[m]$ to maximize the mutual information, i.e.,
\begin{align}
     \max_{P_1,\ldots,P_m} &~MI(P) = H(\mathcal{F}_U) - H(\mathcal{F}_U|\mathcal{F}_Q)& \label{eq:obj}\\ 
    \text{s.t.}\quad &~ C(P_i) \leq T,  ~\forall i\in [m]\\
    &~ |S_i| \leq k_i,  ~\forall i\in [m]. \label{eq:prob:constrs:sensor_number}
\end{align}
\end{problem}

Note that given drop locations SS, the sensor locations in $Q$ are random. If the locations in $Q$ are deterministic, i.e., the sensors fall at the exact locations desired, and if points of interest $U$ are the same as the vertices in $V$, we get the traditional informative path planning problem~\cite{singh2007efficient,binney2010informative}.

Since the locations in $Q$ are themselves random variables, evaluating the probability distribution $Pr(\mathcal{F}_U|\mathcal{F}_Q)$ and its entropy is challenging. In the next section, we discuss how we address this challenge and present the planning algorithm.

\section{Technical Approach}\label{sec:technical_approach}
In this section, we discuss how to evaluate the objective function given in Problem~\ref{pbm:main_problem}. We then propose the planning algorithm to solve the problem.

\subsection{Gaussian Process with Stochastic Drop Locations}

 In order to evaluate the objective function~\eqref{eq:obj}, we need to calculate the entropy of the random variable $(\mathcal{F}_U|\mathcal{F}_Q)$. If the sensor locations in $Q$ were deterministic, this random variable would be a multivariate Gaussian, and its covariance matrix could be used to determine the entropy. However our data is of the form $\{\textbf{q}_i, f(\textbf{q}_i)+\epsilon_i\}_{i=1}^{\sum_j |S_j|}$ and $\textbf{q}_i\sim\mathcal{N}((\bar{\textbf{q}}_i, \Sigma_i))$. Then, since the locations of sensors are independent of each other, the probability distribution $Pr(\mathcal{F}_U|\mathcal{F}_Q)$ is given by integrating the distribution given fixed locations over random sensor locations, i.e., 
 \begin{align*}
     & Pr(\mathcal{F}_U|\mathcal{F}_Q) = \\
     & \int\myDots\int Pr(\mathcal{F}_U|\mathcal{F}_Q,\{\textbf{q}_1, \myDots, \textbf{q}_a\})\prod_{i=1}^{a}\big(Pr(\textbf{q}_i) \big)d\textbf{q}_i\myDots d\textbf{q}_a.
\end{align*}
This distribution is not Gaussian and there is generally no closed
form expression for this integral~\cite{girard2004approximate,muppirisetty2015spatial}.  Existing literature on Gaussian Processes with input uncertainty~\cite{girard2004approximate,muppirisetty2015spatial,mchutchon2011gaussian} resorts to approximations in order to solve this integral. A Monte Carlo approach by drawing samples of $\textbf{q}$ from uncertain location distributions is considered in~\cite{muppirisetty2015spatial}. Taylor expansion about $\bar{\textbf{q}}$ is used in~\cite{mchutchon2011gaussian} to present an approximate method that requires the derivative of the mean of $f$. The Gaussian approximation method~\cite{muppirisetty2015spatial,girard2004approximate,dallaire2011approximate} assumes that the posterior distribution is Gaussian and finds its expected mean and expected covariance by integrating over the uncertainty of the locations $\textbf{q}$. 
For the squared exponential covariance, the expected covariance for normally distributed sensor locations can be analytically computed exactly using the following expression~\cite{girard2004approximate,dallaire2011approximate}.  
\begin{align}
\label{eq:ex_cov}
\begin{split}
    &\Sigma_{QQ}(i,j) = \\
    &\frac{\sigma^2 \text{exp}\Big(  -\frac{1}{2} (\bar{\textbf{q}}_i - \bar{\textbf{q}}_j)^\top (W+\Sigma_i +\Sigma_j)^{-1} (\bar{\textbf{q}}_i - \bar{\textbf{q}}_j)\Big)}{| I+W^{-1}(\Sigma_i+\Sigma_j)(1-\delta_{ij}) |^{1/2}}
\end{split}
\end{align}
Here $\bar{\textbf{q}}_i$ and $\Sigma_i$ are the mean and covariance of the normally distributed sensor location $\textbf{q}_i$ in $Q$, and $W$ is a diagonal matrix where each diagonal element corresponds to a characteristic length scale for the respective input variable.

We use the Gaussian approximation method in this paper because it does not require sampling and is computationally tractable with a simple analytical expression for the covariance matrix. Moreover, since we are planning paths for UAVs offline, before getting any sensor readings, we can use this method to find the mutual information by just using the expected covariance as discussed below.

Since the Gaussian approximation method assumes that the distribution of $\mathcal{F}_U|\mathcal{F}_Q$ is a Gaussian distribution, and because $\mathcal{F}_U$ and $\mathcal{F}_Q$ are jointly Gaussian, the mutual information is given by 
\begin{align}
\label{eq:obj_approx}
\begin{split}
    MI &= H(\mathcal{F}_U)- H(\mathcal{F}_U|\mathcal{F}_Q) \\
    &= H(\mathcal{F}_U) + H(\mathcal{F}_Q) - H(\mathcal{F}_U,\mathcal{F}_Q)\\
    &= \frac{1}{2}\log \Big( \frac{\det(\Sigma_{UU}) \det(\Sigma_{QQ})}{\det(\bar{\Sigma})} \Big),\\
    \end{split}
\end{align}
where
\begin{align*}
    \bar{\Sigma} = \left[\begin{matrix}
    \Sigma_{UU} & \Sigma_{UQ}\\
    \Sigma_{QU} & \Sigma_{QQ}
    \end{matrix}  \right].
\end{align*}
We can use the expression~\eqref{eq:ex_cov} to evaluate $\Sigma_{UQ}(i,j)$ by replacing $\textbf{x}_i$ with the known location of $i^{th}$ point of interest in $U$ and $\Sigma_i$ by the null matrix.

\begin{observation}
The Objective function~\eqref{eq:obj} and the surrogate objective defined in Equation~\eqref{eq:obj_approx} are submodular and monotonically non-decreasing set functions in $S$.
\end{observation}

\begin{algorithm}[ht!]\label{alg:SGA_algorithm}
    \caption{Sequential Greedy Assignment}
    \SetKwInOut{Input}{Input}
    \SetKwInOut{Output}{Output}
    \SetKwProg{Fn}{Function}{:}{}
    \SetKwFunction{SGA}{SGA}
    \Fn{\SGA{$G, \mathcal{T}, v_{s}, f, m, C, K$}}{ 
    \Input{\begin{itemize}
        \item A  graph $G$ representing the environment
        \item Time budget $\mathcal{T}=\{T_1, \ldots, T_m\}$ for each robot 
        \item Starting positions $v_s$ and  objective oracle $g$
        \item \# of robots $m$ and cost function $C$
        \item \# of the sensor for each robot $K=\{k_1, \ldots, k_m\}$
    \end{itemize}}
    \Output{a collection of paths $\{{P}_j\}_{j=1}^m$}
    $G_c \gets$ metric matrix completion of $G$\\ 
    $\mathcal{A} \gets \emptyset$  \\ 
    \For{$j \gets 1$ \KwTo $m$ }{
    ${P}_j \gets GCB(\mathcal{A}, T_j, G_c, g, {C}, v_{s_j}, k_j)$\;
   $ \mathcal{A} \gets \mathcal{A} \cup {P}_j$ 
    }
    \textbf{return} $\{{P}_j\}_{j=1}^m$
    }
    \textbf{end}
\end{algorithm}
\subsection{Planner}


The submodularity and monotonicity of the surrogate objective function allow us to formulate Problem~\ref{pbm:main_problem} as a submodular TOP. However, there is one additional constraint in Problem~\ref{pbm:main_problem} that is not present in standard submodular TOP, that of the number of sensors $k_i$ that each robot is able to deploy. We address this problem using the following observation.
\begin{lemma} \label{lemma:number_of_sensors} 
In a complete graph with $N\geq k_i$ vertices for all $i$, there always exists an optimal solution where the robot $i$'s path consists of no more than $k_i$ vertices, excluding the starting vertex.
\end{lemma}
\begin{proof}
The proof follows by contradiction. Suppose there is an instance where no optimal solution has at most $k_i$ vertices along robot $i$'s path. The robot is allowed to deploy at most $k_i$ sensors. Therefore, there must be one or more vertices along the robot path that no sensor is dropped. Since the graph is a complete metric graph, we can ``shortcut'' such vertices without increasing the cost of the path. Therefore, we can recover a solution that consists of exactly $k_i$ vertices. This is a contradiction proving the original claim.
\end{proof}

With this insight, we present our algorithm (Algorithm \ref{alg:SGA_algorithm}) to solve the Problem \ref{pbm:main_problem}.
We first take the metric completion of the input graph. Recall that for a weighted graph $G(V, E)$, each edge $(u,v) \in E$ is associated with a cost $w(u,v)$. In the preprocessing step, we generate a complete graph $G^{\prime}=(V, E^{\prime})$ using $G$, where the edge cost $w^{\prime}(u,v)$ is defined as the length of the shortest path between $u$ and $v$ in $G$. Then, we sequentially call a subroutine, Generalized Cost-Benefit (GCB), to compute a path for each robot. Compared to the original GCB algorithm~\cite{zhang2016submodular}, in Algorithm \ref{alg:GCB_algorithm}, we add one extra control condition in the while loop to account for the constraint, Eq. \eqref{eq:prob:constrs:sensor_number}, on the number of available sensors using Lemma~\ref{lemma:number_of_sensors}. 

\begin{remark}
The constraints imposed on the paths of UAVs, which limit them to at most $k_i$ vertices and a maximum length of $T_i$ for UAV $i$, can be regarded as a partition matroid constraint. It has been shown in~\cite{goundan2007revisiting} that an $\alpha$-approximate greedy step for submodular maximization over a matroid yields an approximation ratio of $\frac{1}{\alpha+1}$. Hence, given an $\alpha$-approximation algorithm to solve the submodular orienteering problem for a single UAV, Algorithm~\ref{alg:SGA_algorithm} results in a $\frac{1}{\alpha + 1}$ approximation ratio for maximizing Objective~\eqref{eq:obj_approx} for multiple UAVs.  When the paths of all the UAVs are constrained to be of at most $T$ length and $k$ vertices, we get a uniform matroid resulting in $1-\frac{1}{e^\alpha}$ approximation ratio.
\end{remark}

\begin{remark}
A quasi-polynomial time recursive greedy algorithm to solve the single vehicle orienteering problem with submodular rewards is given in~\cite{chekuri2005recursive}, resulting in $\alpha=O\log(\texttt{OPT})$. In this paper we use Generalized Cost Benefit (GCB) algorithm to solve the single UAV problem as it has better runtime than the recursive greedy algorithm~\cite{zhang2016submodular}.
\end{remark}

\begin{algorithm}\label{alg:GCB_algorithm}
    \caption{General Cost-Benefit (GCB)}
    \SetKwInOut{Input}{Input}
    \SetKwInOut{Output}{Output}
    \SetKwProg{Fn}{Function}{:}{}
    \SetKwFunction{FMain}{GCB}
    \Fn{ \FMain{$\mathcal{A}, T , G , f , {C}, v_s, k$} }{
    \Input{
    \begin{itemize}
        \item Set for selected vertices $\mathcal{A}$
        \item Budget $T$, a complete graph $G$
        \item Objective oracle $g$ and starting node $v_s$
        \item Cost function ${C}$
        \item \# of sensor $k$
    \end{itemize}
    }
    \Output{A set of selected vertices $S \subseteq V$}
    $S \gets \{v_s\}$, ~$V^{\prime} \gets V \setminus \{\mathcal{A} \cup v_s \}$\\ 
    \While{$V^{\prime} \neq \emptyset ~\textbf{and} ~ k > 0$}{

    \For{$v \in V^{\prime}$}{
    $\Delta_f^v \gets g(\mathcal{A} \cup S \cup v) - f(\mathcal{A} \cup S)$ \\
    $\Delta_c^v \gets {C}(S \cup v) - {C}(S)$
    }
    $v^* = \argmax \{\frac{\Delta_f^v}{\Delta_c^v} \mid v \in V^{\prime}\}$ \\
    \If{${C}(S \cup v^*) \leq T  $}{
    $k \gets k-1$ \\
    $S \gets S \cup \{v^*\}$
    }
    $V^{\prime} \gets V^{\prime} \setminus v^*$
    }
    \textbf{return} TSP($S$)
    }
    \textbf{end} \\
\end{algorithm}
\section{Evaluation}\label{sec:evaluation}

In this section, we evaluate the performance of our algorithm through a series of numerical experiments. 
We first explain the setup for the simulation. Then, we will show one qualitative example to illustrate the difference between the proposed approach and the baseline. Next, we will quantitatively evaluate the performance of the proposed approaches w.r.t. the uncertainty reduction of PoIs. Moreover, we will show the running time of the proposed algorithm w.r.t. the number of robots.

\subsection{Experimental Setup}
The flying object model used in this study is based on the work described in \cite{yang2019recursive}. This model captures the motion of the sensors, considering the gravity, the sensors' surface area, and the speed of the wind. The sensor mass is set to 10kg. The surface coefficient is 1 and the vertical height is 500m.

We begin by defining the map, ground truth, and wind field, as shown in Fig. \ref{fig:setting}. The map provides labels for all the potential dropping points and PoIs. The ground truth is generated by combining multiple Gaussian functions. Data points sampled from the ground truth are used to learn the kernel function, where we employ the RBF function. The wind field indicates the speed at specific locations on the map. By combining the sensor motion model with the wind field, we can estimate the landing position of the sensors.

Using a given kernel, the Algorithm \ref{alg:SGA_algorithm} is applied to search for a set of sensor dropping locations which is an approximate solution to the main problem. The final sensor locations are determined by sampling from the flying object model with uncertainty. Once the sensor locations are obtained, we can measure the environmental values and compute the posterior of PoIs based on these measurements. 

\begin{figure}[ht]
    \centering
    \subfloat[]{
    \includegraphics[width=0.23 \textwidth]{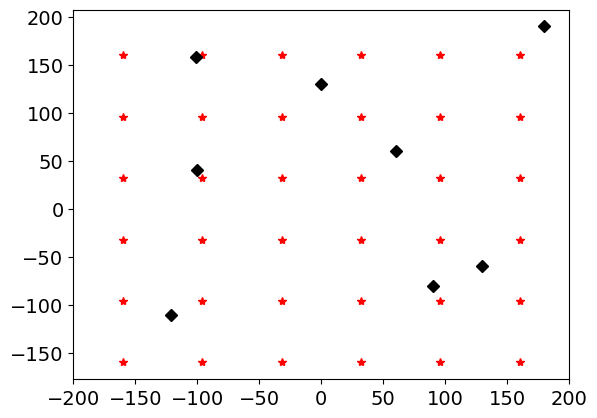}
    \label{fig:Graph}
    }
    \subfloat[]{
    \centering
    \includegraphics[width=0.23 \textwidth]{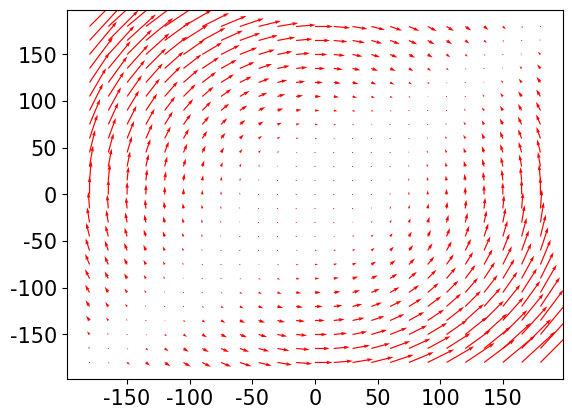}
    \label{fig:wind}
    }
    \caption{Simulation setup. (a) The map of drop location and PoIs. The drop locations are red crosses, which are the target points from where the UAVs can drop sensors. The PoIs are denoted by black diamonds. These PoIs represent specific locations where we are interested in measuring the sensing value. (b) Wind field}
    \label{fig:setting}
\end{figure}

\subsection{An qualitative example}
\begin{figure*}[ht]
    \centering
    \subfloat[]{
    \includegraphics[width=0.30 \textwidth]{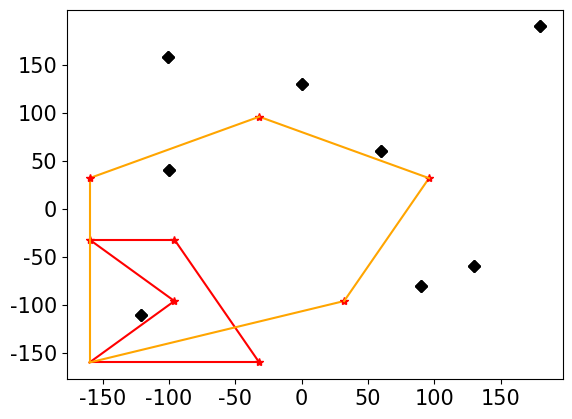 }
    \label{fig:stochastic_path}
    }
    \subfloat[]{
    \includegraphics[width=0.30 \textwidth]{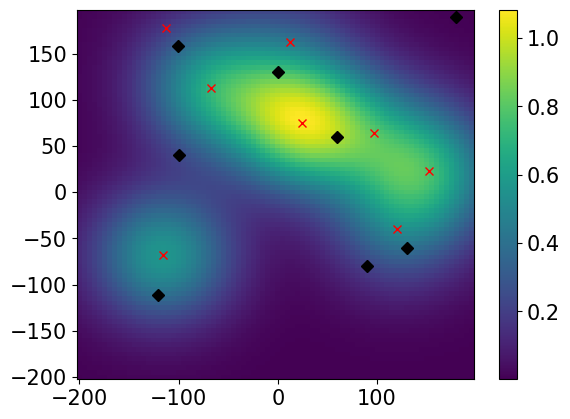 }
    \label{fig:stochastic_pred}
    }
    \subfloat[]{
    \centering
    \includegraphics[width=0.30 \textwidth]{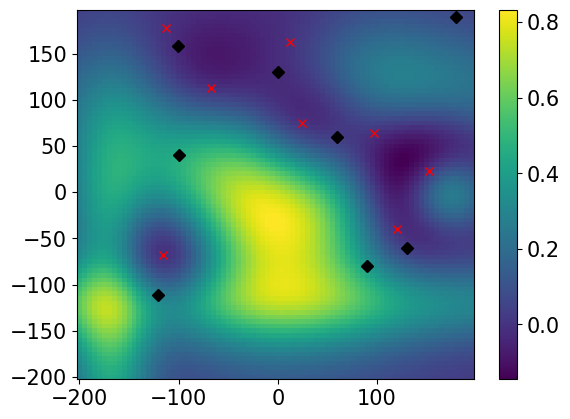}
    
    \label{fig:stochastic_error}
    }
    \caption{ The simulation result of our algorithm with 2 UAVs and 4 sensors each;(a) UAV routes for dropping sensors (Black dot:PoIs, Red star:sensor dropping points, Blue dot: the mean of sensors' landing locations (b) Predicted value (Red cross:the sensors' landing locations ) (c) Error between predicted value and ground truth  }
    \label{fig:Our_appraoch}
\end{figure*}

In the following, we present a comparison between a baseline approach and our proposed method using the defined settings. The experiment focuses on a scenario with two UAVs, where each UAV is equipped with four sensors. The UAVs are allocated a distance budget of 870 units to drop all the sensors along their respective paths.
\subsubsection{Baseline}
In the baseline case (Fig \ref{fig:Baseline}), the UAVs tend to drop a higher number of sensors in areas with a higher concentration of PoIs. The objective is to ensure that each sensor can cover one or more PoIs. However, due to the uncertainty introduced by the wind, the sensors tend to cluster in smaller regions. As a result, the four sensors located around coordinates (0,100) are only capable of accurately estimating two PoIs' value, while the remaining PoIs are not sufficiently covered. This can be observed in Fig \ref{fig:Baseline}, where the two PoIs in the lower right corner exhibit a significantly higher error of estimation.
\subsubsection{Our Approach}
Our approach, on the other hand, considers the impact of wind uncertainty and prefers to drop sensors in a wider area. As shown in Fig. \ref{fig:Our_appraoch}, the wind blows the sensors to a broader coverage area, allowing them to reach and cover more PoIs. This broader coverage results in a significant reduction in the error of PoI estimation compared to the baseline case. Additionally, it is worth noting that the areas where the sensors are dropped but do not have high concentration of PoIs exhibit high error rates. This demonstrates the effectiveness of our approach in adapting to the wind uncertainty and achieving better coverage of the target area.
\begin{figure*}[ht]
    \centering
    \subfloat[ UAV routes for dropping sensors]{
    \includegraphics[width=0.30 \textwidth]{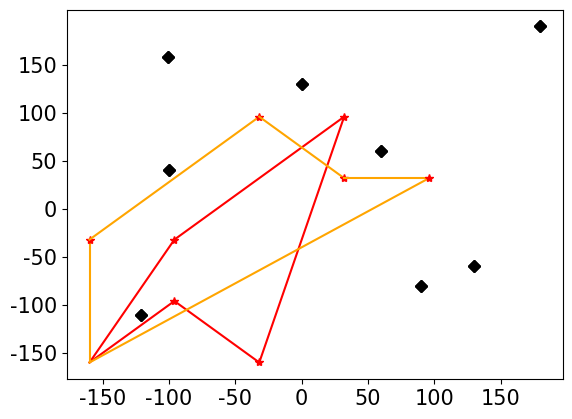 }
    \label{fig:deterministic_path}
    }
    \subfloat[Predicted value]{
    \includegraphics[width=0.30 \textwidth]{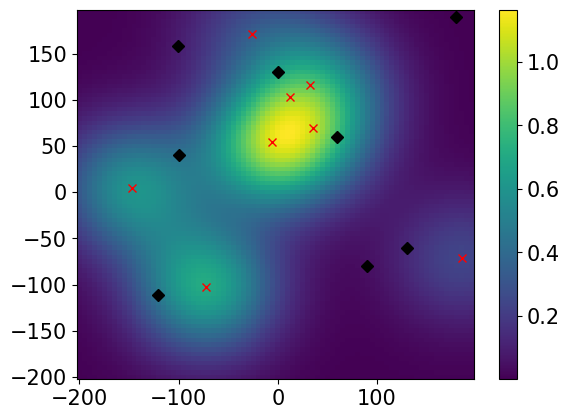 }
    \label{fig:deterministic_pred}
    }
    \subfloat[Error between predicted value and ground truth]{
    \centering
    \includegraphics[width=0.30 \textwidth]{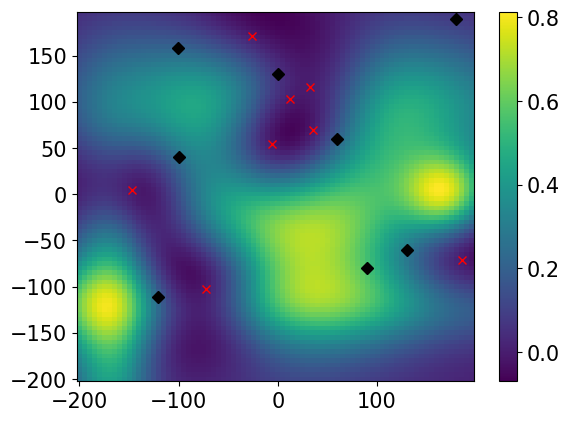}
    
    \label{fig:deterministic_error}
    }
    \caption{The simulation result of baseline algorithm with 2 UAVs and 4 sensors each;(a)UAV routes for dropping sensors (Black dot: PoIs, Blue star: sensor dropping points, Blue dot: the mean of sensors' landing locations (b) Predicted value (Red cross: the sensors' landing locations ) (c) Error between the predicted value and ground truth}
    \label{fig:Baseline}
\end{figure*}

\subsection{Comparisons with Baselines}
In this section, we compare the MSE of three different approaches across three different scenarios.  The MSE is computed as the sum of the square of the difference between the posterior of the PoIs and the ground truth values of the PoIs. In the first two scenarios, we assume that the wind speed is uniform and the variance of landing location is the same for all dropping nodes. In the first scenario, the final location of a sensor follows a Gaussian distribution with a variance of 900. Two UAVs are deployed, with each carrying 4 sensors. In the second scenario, the final location of a sensor follows a Gaussian distribution with a variance of 820. Two UAVs are deployed, with each carrying 3 sensors. In both of these scenarios, our approach demonstrates approximately a 10\% improvement in MSE compared to the baseline approach. The random selection approach, on the other hand, results in an MSE of 1.

The third scenario introduces non-uniform uncertainty w.r.t. the drop point location, where the variance is a function of the non-uniform wind speed. Once again, our approach consistently outperforms the baseline approach, achieving a 12\% improvement in MSE. These results highlight the effectiveness of our approach in mitigating the impact of uncertainty in different scenarios and achieving more accurate sensor placements.

\begin{figure}[ht]
    \centering
    \includegraphics[width=0.25 \textwidth]{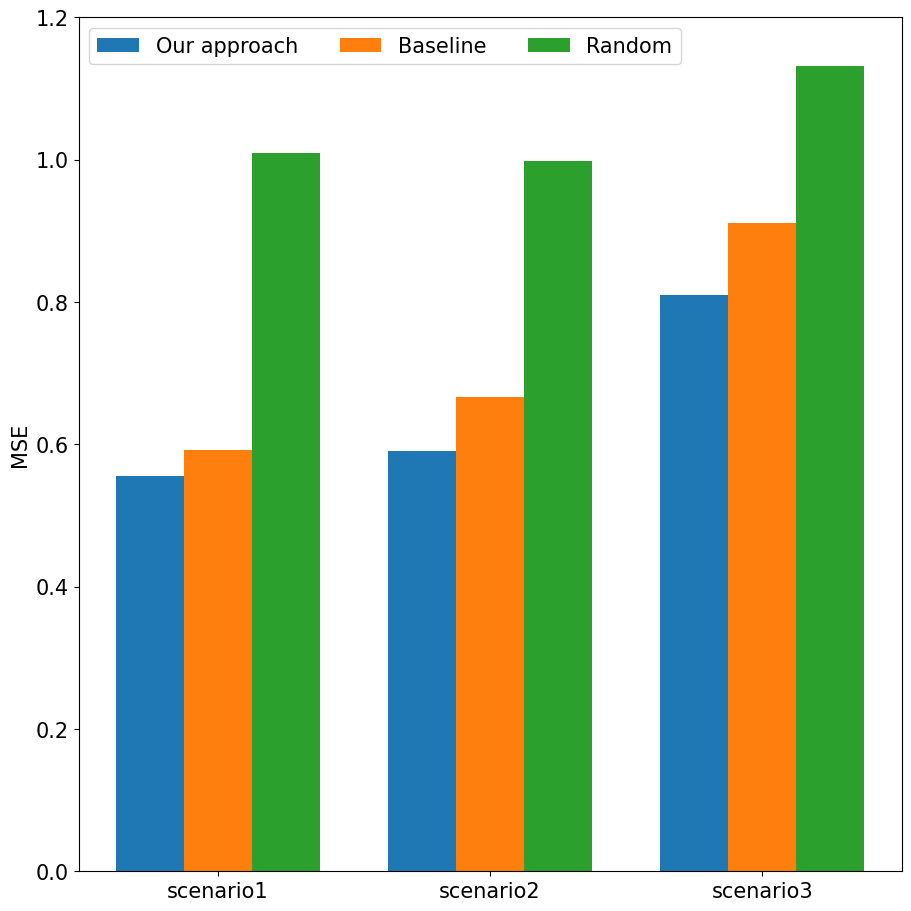}
    \label{fig:MSE}
    
    \caption{The MSE of different approaches and scenarios.}
\end{figure}
\subsection{Running Time}
Lastly, we demonstrate the scalability of our approach. In comparison to the baseline approach, our approach may have a slightly longer running time in each scenario. However, both approaches grow polynomially in run time with the number of sensors per UAV.

To further evaluate the computational performance, we also simulated a brute-force approach. The brute-force approach generates all possible combinations of sensor dropping points within the budget constraint and selects the set with the highest objective value. The runtime of the brute-force approach grew exponentially, taking hours to days to complete due to the factorial computation of all possible combinations. This stark contrast highlights the effectiveness and efficiency of our approach in finding nearly optimal solutions for sensor placement in a timely manner.
\begin{figure}[ht]
    \centering
    \includegraphics[width=0.30 \textwidth]{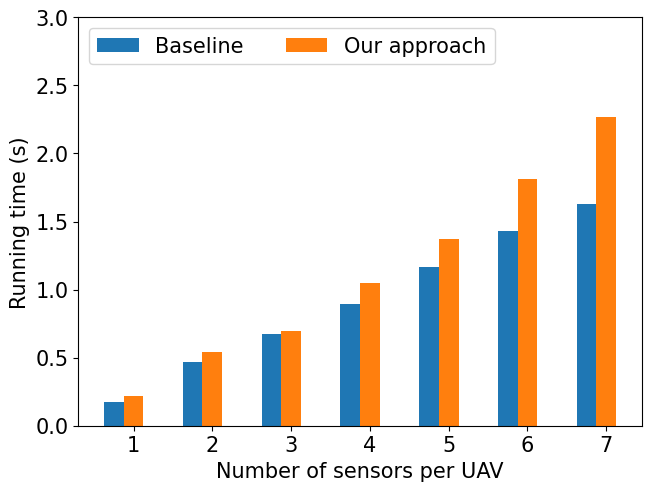}
    \label{fig:running_time}
    
    \caption{The running time with different numbers of sensors per agent in the scenarios of 2 UVAs.}
\end{figure}
\section{Conclusion}
This paper studies the problem of routing a team of UAVs to drop sensors to reduce the uncertainty of PoIs. The problem is formulated as a variant of TOP. To reduce the computational cost in the evaluation of the objective, we propose one surrogate objective with closed-form expression based on Gaussian approximation. A heuristic algorithm (SGA) is proposed to solve the relaxed problem with the surrogate objective. The formulation and the algorithm are validated in numerical simulation.



\newpage

\bibliographystyle{IEEEtran}
\bibliography{IEEEabrv, main}

\end{document}